\documentclass[10pt,twocolumn,letterpaper]{article}

\usepackage{wacv}
\usepackage{times}
\usepackage{epsfig}
\usepackage{graphicx}
\usepackage{amsmath}
\usepackage{amssymb}

\usepackage{amsmath}
\usepackage{amsthm}
\usepackage{dsfont}

\newtheorem{theorem}{Proposition}
\usepackage{algorithm}
\usepackage{algorithmic}
\usepackage{booktabs}

\wacvfinalcopy 

\def\wacvPaperID{216} 

\ifwacvfinal\fi
\begin{document}

\title{ViP: Virtual Pooling for Accelerating CNN-based \\ Image Classification and Object Detection}

\author{Zhuo Chen \hspace{1cm} Jiyuan Zhang \hspace{1cm}  Ruizhou Ding\\
Carnegie Mellon University\\
{\tt\small zhuoc1, jiyuanz, rding@andrew.cmu.edu}
\and
Diana Marculescu \\
University of Texas at Austin \& Carnegie Mellon University\\
{\tt\small dianam@utexas.edu \& dianam@cmu.edu}
}

\maketitle
\ifwacvfinal\thispagestyle{empty}\fi

\begin{abstract}
   In recent years, Convolutional Neural Networks (CNNs) have shown superior capability in visual learning tasks. While accuracy-wise CNNs provide unprecedented performance, they are also known to be computationally intensive and energy demanding for modern computer systems. In this paper, we propose Virtual Pooling (ViP), a model-level approach to improve speed and energy consumption of CNN-based image classification and object detection tasks, with a provable error bound. We show the efficacy of ViP through experiments on four CNN models, three representative datasets, both desktop and mobile platforms, and two visual learning tasks, \textit{i.e.}, image classification and object detection. For example, ViP delivers \textbf{2.1x} speedup with less than \textbf{$1.5\%$} accuracy degradation in ImageNet classification on VGG16, and \textbf{1.8x} speedup with $0.025$ mAP degradation in PASCAL VOC object detection with Faster-RCNN. ViP also reduces mobile GPU and CPU energy consumption by up to \textbf{55\%} and \textbf{~70\%}, respectively. As a complementary method to existing acceleration approaches, ViP achieves \textbf{1.9x} speedup on ThiNet leading to a combined speedup of \textbf{5.23x} on VGG16. Furthermore, ViP provides a knob for machine learning practitioners to generate a set of CNN models with varying trade-offs between system speed/energy consumption and accuracy to better accommodate the requirements of their tasks. Code is available at https://github.com/cmu-enyac/VirtualPooling. 
\end{abstract}
\vspace{-20pt}
\section{Introduction}
\label{introduction}
\vspace{-5pt}
\begin{figure}[htb!]
	\centering
	\includegraphics[width=.48\textwidth]{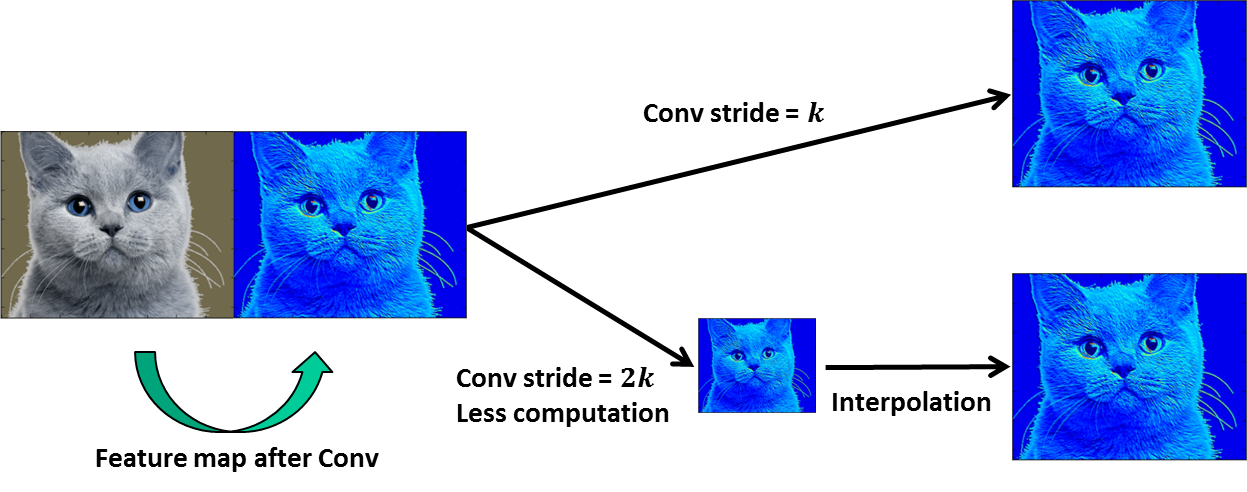}
	\vspace{-15pt}
	\caption{Illustration of virtual pooling \cite{matviz}. By using a larger stride, we save computation in conv layers and, to recover the output size, we use linear interpolation which is fast to compute.}
	\label{fig:vip-viz}
	\vspace{-15pt}
\end{figure}

Deep Convolutional Neural Networks (CNNs) have gained tremendous traction in recent years thanks to their outstanding performance in visual learning tasks, \textit{e.g.}, image classification and object detection \cite{vgg16,chen2018understanding,faster-rcnn}. However, CNNs are often considered very computationally intensive and energy demanding \cite{lebedev2014speeding,han2015learning,cai2017neuralpower}. With the prevalence of mobile devices, being able to run CNN-based visual tasks efficiently, in terms of both speed and energy, becomes a critical enabling factor of various important applications, \textit{e.g.}, augmented reality, self-driving cars, Internet-of-Things, \textit{etc}, which all heavily rely on fast and low energy CNN computation. To alleviate the problem, engineers and scientists proposed various solutions, including sparsity regularization, connection pruning, model quantization, low rank approximation, \textit{etc}. In this work, we propose an complementary approach, called \textit{Virtual Pooling} (ViP), which takes advantage of pixel locality and redundancy to reduce the computation cost originating from the most computationally expensive part of CNN: convolution layers (conv layers). As illustrated in Fig.\ref{fig:vip-viz}, ViP reduces computation cost by computing convolution with a larger (2x) stride size. While naturally this operation quickly shrinks the output feature map, and thus can only be done a few times before the image vanishes, we overcome this problem by recovering the feature map via \textit{linear interpolation} with provable error bound. The succeeding layer hence observes the same size of input with or without ViP, and no architectural change is needed. Our experimental results on different CNN models and tasks show that we can achieve \textbf{2.1x} speedup with \textbf{$1.5\%$} accuracy degradation in image classification, compared to the 1.9x speedup with \textbf{$2.5\%$} drop from the prior work \cite{perforatedcnns}, and \textbf{1.8x} speedup with $0.025$ mAP drop in object detection. 

\section{Related Work and Contributions}
\label{related}
There are several prior works targeting CNN acceleration \cite{han2015learning,wen2016learning,perforatedcnns,lightnn}. Model compression \cite{han2015learning,wen2016learning,he2018soft,he2018amc} is a popular approach of reducing CNN memory requirement and runtime via weight pruning. \cite{han2015learning} proposed to prune connections and finetune the network progressively which results in high compression rate. However, due to the non-structured sparsity generated by this method, it also needs specialized hardware to realize high speedup \cite{eie}. In light of this, \cite{wen2016learning} used group lasso to generate structured sparsity and speed up CNNs on general-purpose processors. 

CNN model binarization or quantization methods \cite{binaryconnect,binarynet,lightnn,wu2016quantized,incre-net-quant,ternary-quant} quantize CNN weights and/or activations into low-precision fewer-bit representations. Thereafter, they are able to both reduce memory cost and speedup computation by using efficient hardware units. \cite{binaryconnect} uses binary weights rather than continuous-valued weights in CNN models, which is not only able to save memory space, but also greatly speedup convolution via replacing multiply-accumulate operations by simple accumulations. Ding \textit{et al.,} \cite{lightnn} reduces the number of bits of CNN weights through its binary representation, which can be sped up by using shift-add operation rather than expensive multipliers on hardware. \cite{binarynet,xnor-net} further quantize the CNN intermediate activations, resulting in both binary weight and input, which can be further accelerated via efficient XNOR operation.

Low rank approximation methods \cite{jaderberg2014speeding,lebedev2014speeding,denton2014exploiting} speed up convolution computation by exploiting the redundancies of the convolutional kernel using low-rank tensor decompositions. The original conv layer is then replaced by a sequence of conv layers with low-rank filters, which have a much lower total computational cost. \cite{jaderberg2014speeding} exploit cross-channel or filter redundancy to construct rank-one basis of filters in the spatial domain. \cite{lebedev2014speeding} use non-linear least squares to compute a low-rank CP-decomposition of the filters into fewer rank-one tensors and then finetune the entire network. 

The closest work to ours is PerforatedCNNs \cite{perforatedcnns} which, inspired by the idea of loop perforation \cite{sidiroglou2011managing}, reduces the computation cost in conv layers by exploiting the spatial redundancy. Nevertheless, PerforatedCNNs use a dataset dependent method to generate an irregular output mask that determines which neuron should be computed exactly. In addition, PerforatedCNNs need a mask at runtime (hence introducing overhead) to determine the value for interpolation. In contrast, ViP only depends on the intermediate activations of the CNN layer without extra parameters. Finally, PerforatedCNNs also consider the use of a pooling-structured mask, but that can only be applied to the layers immediately preceding a pooling layer; also, the associated interpolation method is nearest neighbor. In contrast, our method can be applied to any conv layer in the network. Furthermore, we show that the ViP method achieves higher speedup with lower accuracy degradation. To the best of our knowledge, our work makes the following contributions:
\vspace{-5pt}
\begin{enumerate}
	\setlength\itemsep{0em}
	\item We are the first to propose and implement the Virtual Pooling (ViP) method with provable error bound. ViP is independent of the dataset and can be applied to accelerate any conv layer. 
	\item Plug-and-play: ViP is a self-contained custom layer. Without modifying the deep learning framework, it works simply by doubling the stride of the conv layer and inserting the ViP layer after it. 
	\item ViP can be combined with existing model acceleration methods, \textit{e.g.}, model compression, quantization, \textit{etc}., to squeeze more performance out of the CNN models.
	\item More than providing a single CNN configuration, ViP generates a set of models with varying speedup/energy and accuracy trade-offs from which a machine learning practitioner can select for the task at hand.
	\item Most CNN acceleration techniques consider only the image classification task, while they lack evidence on how their performance may translate to the object detection task, which has its own unique properties. In this work, we conduct experiments to show that ViP also works well under the state-of-the-art faster-rcnn object detection framework. 
\end{enumerate}
\vspace{-5pt}
The remainder of this paper is organized as follows. Section \ref{Method} introduces the details of the virtual pooling method. In Section \ref{results}, we conduct extensive experiments with different CNN models on both desktop and mobile platforms, and we apply ViP to speed up both image classification and object detection tasks. Finally, we conclude our work in Section \ref{conclusion}.
\vspace{-15pt}

\section{Methodology}
\label{Method}
Virtual Pooling (ViP) relies on the idea of reducing CNN computation cost by taking advantage of pixel spatial locality and redundancy. CNNs are often comprised of multiple conv layers interleaved with pooling layers. Pooling layers are considered essential for reducing spatial resolution such that computation cost is reduced and robustness to small distortions in images is enhanced. However, the widely-used stride-two non-overlapping pooling method \cite{vgg16,resnet} reduces image size by half in each of the two dimensions, and thus quickly shrinks the image. As a result, the maximum number of pooling operations that can be done in a CNN is limited by the size of the input image. For example, an input image of size $224*224$ is shrunk to size $7*7$ after only five pooling layers in VGG16, while the current state-of-the-art CNNs usually have several tens to hundreds of layers \cite{vgg16,resnet}. There is an opportunity to reduce computation further if we can bridge the gap between the number of pooling-like operations we can do and the number of layers in the network. 

\subsection{ViP Layer}
To this end, we propose ViP, a method that can maintain the output size of each layer, while using a larger-stride convolution. Consequently, we can have as many ViP layers as possible while not encountering the problem of diminishing image size in the real pooling operation. While it is possible to increase the stride of an early layer and remove a later pooling layer to achieve a similar effect, our experiments show that ViP is consistently better than pooling removal with 1.42\% higher accuracy on average. Furthermore, this method can only reduce computation in consecutive conv layers prior to pooling, while ViP works in any order (as we will show later, accuracy sensitivity is non-monotonic with the network layer) which gives a better accuracy-speedup curve. As shown in Fig.\ref{fig:vip-viz}, ViP saves computation by performing a larger stride convolution in the layer before ViP, and then recovers the output size by linear interpolation which is very computationally efficient. For example, we can double the stride of all conv layers in VGG16 to reduce computation, while the succeeding layer observes the input of exactly the same size after linear interpolation. The theoretical speedup of this approach is 4x as we halve the number of convolutions in two dimensions. Though transposed convolution (Deconv) \cite{long2015fully} is also an upsampling method, to speedup the network, its overhead must be sufficiently small so it does not offset the reduced latency. In the supplementary material we show that ViP is very efficient and its computation is only 0.016\% of Deconv.

To be more specific, let's use $\mathcal{I}$ to denote the input to conv layer and $\mathcal{O}$ to denote the output. Without loss of generality, although $\mathcal{I}$ and $\mathcal{O}$ are often four-dimensional, we omit the first dimension of batch index because ViP applies to all images in the batch independently, and therefore, $\mathcal{I}_{c,h,w}$ and $\mathcal{O}_{c,h,w}$ both have three dimensions: channel $c \in [1,C]$, height $h \in [1, H]$, and width $w \in [1, W]$. We consider convolution filters, $\mathcal{W}_{c',c,m,n}$, with the same height and width with odd values $M$ as are commonly used in CNNs \cite{vgg16,resnet}, and with $c'$ representing the index of the filter. For the purpose of simplicity, we further assume $H$ and $W$ are even numbers, \textit{e.g.}, input image size of ImageNet is usually $224*224$, and in the case of odd numbers, we have special cases only on the boundaries of the image that are easy to deal with. Furthermore, we use $\mathcal{O}_{c',h,w}^{Orig}$ to represent the output of the original stride-$s$ convolution without ViP, and $\mathcal{O}_{c',h,w}^{ViP}$ to denote the output of using ViP method, \textit{i.e.}, the output of stride-$2s$ convolution plus linear interpolation. A smaller $||\mathcal{O}_{c',h,w}^{Orig}-\mathcal{O}_{c',h,w}^{ViP}||_2$ indicates a smaller perturbation of the truth output and hence, less accuracy degradation for the ViP method. According to the definition of convolution:
\vspace{-10pt}
\begin{equation}
\small 
\mathcal{O}_{c',h,w}^{Orig} = \sum_{c=1}^{C} \sum_{m,n=-\lfloor \frac{M}{2} \rfloor}^{\lfloor \frac{M}{2} \rfloor} \mathcal{I}_{c, s\cdot h-m, s\cdot w-n} * \mathcal{W}_{c',c,m,n}
\vspace{-5pt}
\end{equation}
If we double the stride, we have an output with reduced size:
\vspace{-10pt}
\begin{equation}
\small
\mathcal{O}_{c',h,w}^{Red} = \sum_{c=1}^{C} \sum_{m,n=-\lfloor \frac{M}{2} \rfloor}^{\lfloor \frac{M}{2} \rfloor} \mathcal{I}_{c, 2s\cdot h-m, 2s\cdot w-n} * \mathcal{W}_{c',c,m,n}
\end{equation}
For ease of explanation, we use an auxiliary function $\mathcal{O}_{c',h,w}^{Zero}$ which is zero-spaced to enlarge $\mathcal{O}_{c',h,w}^{Red}$ to the same size of $ \mathcal{O}_{c',h,w}^{Orig}$ in the following way:
\begin{equation}\label{eq:zero-reduce}
\mathcal{O}_{c',h,w}^{Zero} = \begin{cases}
\mathcal{O}_{c',h/2,w/2}^{Red} &\text{$h,\ w$ are even numbers}\\
$0$ &\text{Otherwise}
\end{cases}
\end{equation}
We approximate the output with the ViP method $\mathcal{O}_{c',h,w}^{ViP}$ by using the mean of its immediate non-expanding-zero neighbors (including itself, if computed exactly) in $\mathcal{O}_{c',h,w}^{Zero}$:
\begin{equation}
\mathcal{O}_{c',h,w}^{ViP} = \frac{\sum_{m,n=-1}^{1} \mathcal{O}_{c',h+m,w+n}^{Zero}}{\sum_{m,n=-1}^{1} \mathds{1}(\mathcal{O}_{c',h+m,w+n}^{Zero} \neq 0)}
\end{equation}
This is actually a convolution with $3*3$ filters, but with variable weight values depending on the number of non-expanding-zero neighbors.
We can simplify the above computation by considering four cases similar to Eq.\ref{eq:zero-reduce}:
\begin{equation} \label{eq:vip-reduce}
\small
\mathcal{O}_{c',h,w}^{ViP} = \begin{cases}
\mathcal{O}_{c',h/2,w/2}^{Red} &\text{$h$ even, $w$ even}\\
\frac{1}{2}(\mathcal{O}_{c',\lfloor h/2 \rfloor,w/2}^{Red}+\mathcal{O}_{c',\lceil h/2 \rceil,w/2}^{Red}) &\text{$h$ odd, $w$ even} \\
\frac{1}{2}(\mathcal{O}_{c',h/2,\lfloor w/2 \rfloor}^{Red}+\mathcal{O}_{c',h/2,\lceil w/2 \rceil}^{Red}) &\text{$h$ even, $w$ odd} \\

\frac{1}{4}(\sum_{\substack{h=\lfloor h/2 \rfloor \ or\ \lceil h/2 \rceil\\ 
		w=\lfloor w/2 \rfloor\ or\ \lceil w/2 \rceil}} \mathcal{O}_{c',h,w}^{Red}) &\text{$h$ odd, $w$ odd} 
\end{cases}
\end{equation}
The above equations are embarrassingly parallel and hence fast to compute on GPU. We implemented our custom ViP layer based on Eq.\ref{eq:vip-reduce}.

We can further provide an error bound, considering the case where we apply ViP to layer $l_s$. 
\vspace{-5pt}
\begin{theorem}
	Assume the output of layer $l_s$ (hence input to layer $l_{s+1}$), $\mathcal{O}^{(l_s)}$, is L-Lipschitz continuous \cite{li2017training} on height and width dimensions $(h,w)$, \textit{i.e.}, \\ $|\mathcal{O}^{(l_s)}_{c,h_1,w_1}-\mathcal{O}^{(l_s)}_{c,h_2,w_2}|\le L||(h_1,w_1)-(h_2,w_2)||_2$, for $\forall h_1,h_2\in[1,H], w_1,w_2\in[1,W]$. \\ Assume that $\forall c',l$, the $c'$-th convolutional filter of the $l$-th layer, denoted as $\mathcal{W}^{(l)}_{c'}$, has a bounded $l_2$-norm: $||\mathcal{W}^{(l)}_{c'}||_2=\sqrt{mean^2(\mathcal{W}^{(l)}_{c'})+std^2(\mathcal{W}^{(l)}_{c'})}\le B^{(l)}$. Then, the $l_2$-norm of the output error is bounded by:
	\begin{equation} \label{eq:bound}
	\begin{split}
	& ||\mathcal{O}^{(l_e)ViP} - \mathcal{O}^{(l_e)Orig}||_2 \\
	\le & \sqrt{2}L\sqrt{C'^{(l_e)}H^{(l_e)}W^{(l_e)}}\prod_{l=l_s+1}^{l_e}\sqrt{C^{(l)}}M^{(l)}B^{(l)},\\
	\end{split}
	\end{equation}
	where $C^{(l)}$ and $M^{(l)}$ are the number of input channels and kernel size of the $l$-th layer, respectively, and $C'^{(l_e)}$ is the number of output channels of the $l_e$-th layer.
\end{theorem} 

\begin{proof}
	\vspace{-5pt}
	Deferred to Appendix.
	\vspace{-5pt}
\end{proof}
\vspace{-5pt}
If $\forall l>l_s, \sqrt{C^{(l)}}M^{(l)}B^{(l)} > 1$, the upper-bound will keep increasing when the output goes through multiple layers. This indicates that earlier ViP layers with more succeeding layers may have a bigger impact on the final output of the network and hence higher accuracy drop without finetuning. This actually reflects the intuition that perturbations from early layers will lead to higher error on the output as they propagate through the network. We will see this effect in both VGG16 (Fig.\ref{fig:sens-vgg}) and ResNet-50 (Fig.\ref{fig:sens-res}).
\begin{algorithm}[tb]
	\small
	\caption{Virtual Pooling (ViP)}
	\label{alg:vip}
	\begin{algorithmic}[1]
		\STATE {\bfseries Input:} model $Net$
		\STATE {\bfseries Output:} ViP model $ViPNET$, Accuracy $ViPA$, Runtime $ViPR$
		
		\STATE // Sensitivity analysis
		\STATE $i = 0$
		\STATE $ViPLayers=$[]
		\FOR{$c$ {\bfseries in} $Net.ConvLayers$}
		
		\STATE $A_c=$ evaluate($Net.ViP(c)$)
		\STATE $ViPLayers.append((c, A_c))$
		
		\ENDFOR
		\STATE $ViPLayers.sorted(key=A_c, 'descending')$

		\STATE //Progressively interpolate and finetune
		\STATE $ViPA =$ [], $ViPR =$ []
		\FOR {$j=0:len(ViPLayers)$}
		\STATE $Net = Finetune(Net.ViP(0:j))$
		\STATE $ViPA.append(evaluate(Net))$
		\STATE $ViPR.append(time(Net))$
		\ENDFOR
		
		\STATE Return $ViPNET=Net$, $ViPA$, $ViPR$
	\end{algorithmic}
\end{algorithm}
\vspace{-5pt}
\subsection{ViP Algorithm}
While speeding up CNNs can be achieved with ViP, it may also lead to some accuracy drop since interpolation is a method of approximation. Therefore, we propose the following procedure, as shown in Algorithm \ref{alg:vip}, as part of the ViP method to reduce the accuracy degradation while maximizing the speedup we can achieve. We first do sensitivity analysis to detect which layers are less sensitive, in terms of the accuracy of the network, to ViP (Line 6-9). For each conv layer $c$, we insert ViP after it, and evaluate the network accuracy $A_c$ without finetuning. The sensitivity is measured as the accuracy drop with respect to the original accuracy. Lower $A_c$ leading to a larger accuracy drop means that the layer is more sensitive to ViP, so we sort $A_c$ in descending order as shown in Line 10.
We insert ViP layer after ReLU which follows the conv layer, as both our experiments and prior work \cite{perforatedcnns} show that inserting after ReLU gives better results. Our intuition is that by applying ViP before ReLU, we obtain less activations than the original without ViP and the network becomes less likely to identify smaller activation regions. Therefore, throughout the paper, whenever we mention inserting ViP after conv layer, we mean inserting it after the ReLU layer that follows immediately. 

Based on the sorted sensitivity $ViPLayers$, we insert ViP layers progressively, and finetune the network to achieve a set of models with different speedup-accuracy trade-offs (Line 13-17). For example, we add ViP after $ViPLayers[0]$, finetune the model and obtain the first model, and then we add ViP after both $ViPLayers[0]$ and $ViPLayers[1]$, finetune the model and obtain the second model, and so on so forth. In this fashion, we will generate $len(ViPLayers)$ models ($len(ViPLayers)$ is the total number of conv layers that we apply ViP to), all with different accuracy and runtime. 
However, repetitively finetuning the model $len(ViPLayers)$ times can be quite time-consuming, especially for large CNN models. To alleviate this problem, we conduct grouped finetuning, in which we insert several ViP layers at a time (still based on sensitivity values). This results in fewer rounds, and hence less time, of finetuning, and both per-layer and grouped finetuning methods can generate different accuracy-speedup trade-offs for the baseline model. An example of applying ViP to applications, such as a face detector in a mobile camera system is given in the supplementary material.
\vspace{-5pt}
\section{Experimental Results}
\label{results}
In this section, we first describe the hardware and software setup of our experiments, and then present results to show the effectiveness of ViP method under: 
\vspace{-5pt} 
\begin{enumerate}
	\setlength\itemsep{0em}
	\item Four CNN models: VGG16 \cite{vgg16}, ResNet-50 \cite{resnet}, All-CNN \cite{allcnn}, Faster-RCNN with VGG16 backbone \cite{faster-rcnn}.
	\item Three datasets: ImageNet \cite{imagenet}, CIFAR-10 \cite{cifar10}, PASCAL-VOC \cite{pascal}.
	\item Two hardware platforms: Desktop and Mobile.
	\item Two visual learning tasks: Image classification and object detection.
\end{enumerate}
\vspace{-5pt}
\subsection {Experimental Setup}
Throughout the experiments, we use Caffe \cite{caffe} as our deep learning platform since its correctness has been validated by numerous research works. For fast training and inference, we implement a self-contained custom ViP layer in CUDA and integrate it into Caffe. The ViP layer inserts interpolated points between both columns and rows. The row and column size is doubled after interpolation and the resultant image size is enlarged four times. 
Interpolation is performed independently on points. This process is therefore embarrassingly parallel and can be easily accelerated by GPU. Each thread launched by the CUDA kernel processes one interpolated element. The thread block dimension order from fastest- to slowest-changing are column, row, channel, and batch to match the data layout in Caffe. Based on their position in the interpolated image, the points to be interpolated are classified into four types and estimated using Eq.\ref{eq:vip-reduce}.

CNNs are now widely deployed and used in both cloud services and mobile phones, therefore we experiment with both a high-end desktop machine and a mobile platform with low power and energy profile. The detailed configurations are shown in Table \ref{table:config}. The desktop computer is equipped with high-end Intel Core-i7 CPU and Nvidia Titan X GPU, while the mobile platform is the Jetson TX1 comprised of efficient Quad-core ARM A57 CPU and Nvidia GPU with Maxwell architecture and 256 CUDA cores.

\renewcommand{\arraystretch}{0.8}
\begin{table}[]
	\caption{System Configurations for desktop and mobile platforms.}
	\label{table:config}
	\vskip 0.15in
	\begin{center}
		\begin{small}
			\begin{sc}
				\begin{tabular}{lr}
					\toprule
					\multicolumn{2}{c}{Desktop} \\
					\midrule
					CPU/Main memory    & Intel Core-i7 / 32GB \\
					GPU/Memory   	   & Nvidia Titan X / 12GB \\
					DL platform & Caffe on Ubuntu 14 \\
					\midrule
					\multicolumn{2}{c}{Mobile (Nvidia Jetson TX1)} \\
					\midrule
					CPU/Main memory    & Quad ARM A57 / 4GB \\
					GPU		   & Nvidia Maxwell Arch \\
					DL platform & Caffe on Ubuntu 14 \\
					\bottomrule
				\end{tabular}
			\end{sc}
		\end{small}
	\end{center}
	\vskip -0.1in
	\vspace{-10pt}
\end{table}
\begin{figure*}[htb!]
	\centering
	\includegraphics[width=0.8\linewidth,trim=4 4 4 4, clip]{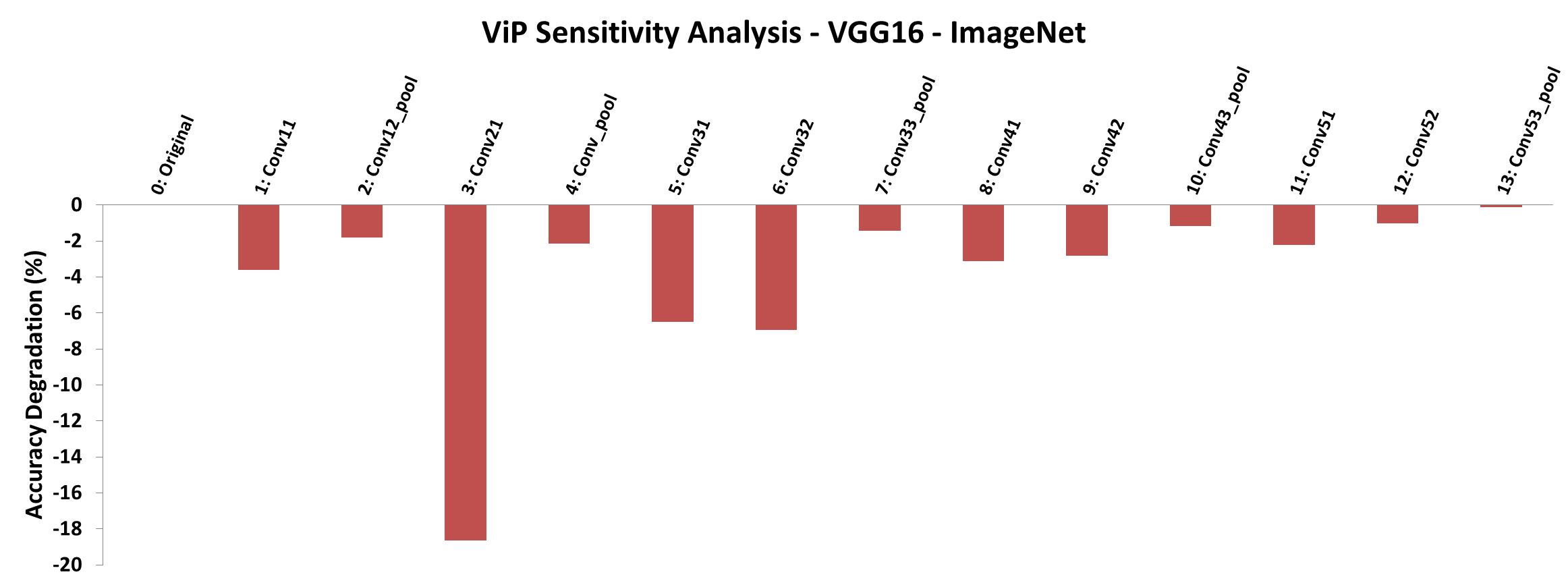}

	\caption{ViP sensitivity analysis of VGG16 model under ImageNet dataset. For each of the conv layers, we insert ViP immediately after it, and evaluate the network accuracy without finetuning. The sensitivity is measured as the accuracy degradation.}
	\label{fig:sens-vgg}
	\vspace{-10pt}
\end{figure*}
\begin{figure}[htb!]
	\centering
	\includegraphics[width=\linewidth,trim=4 4 4 4, clip]{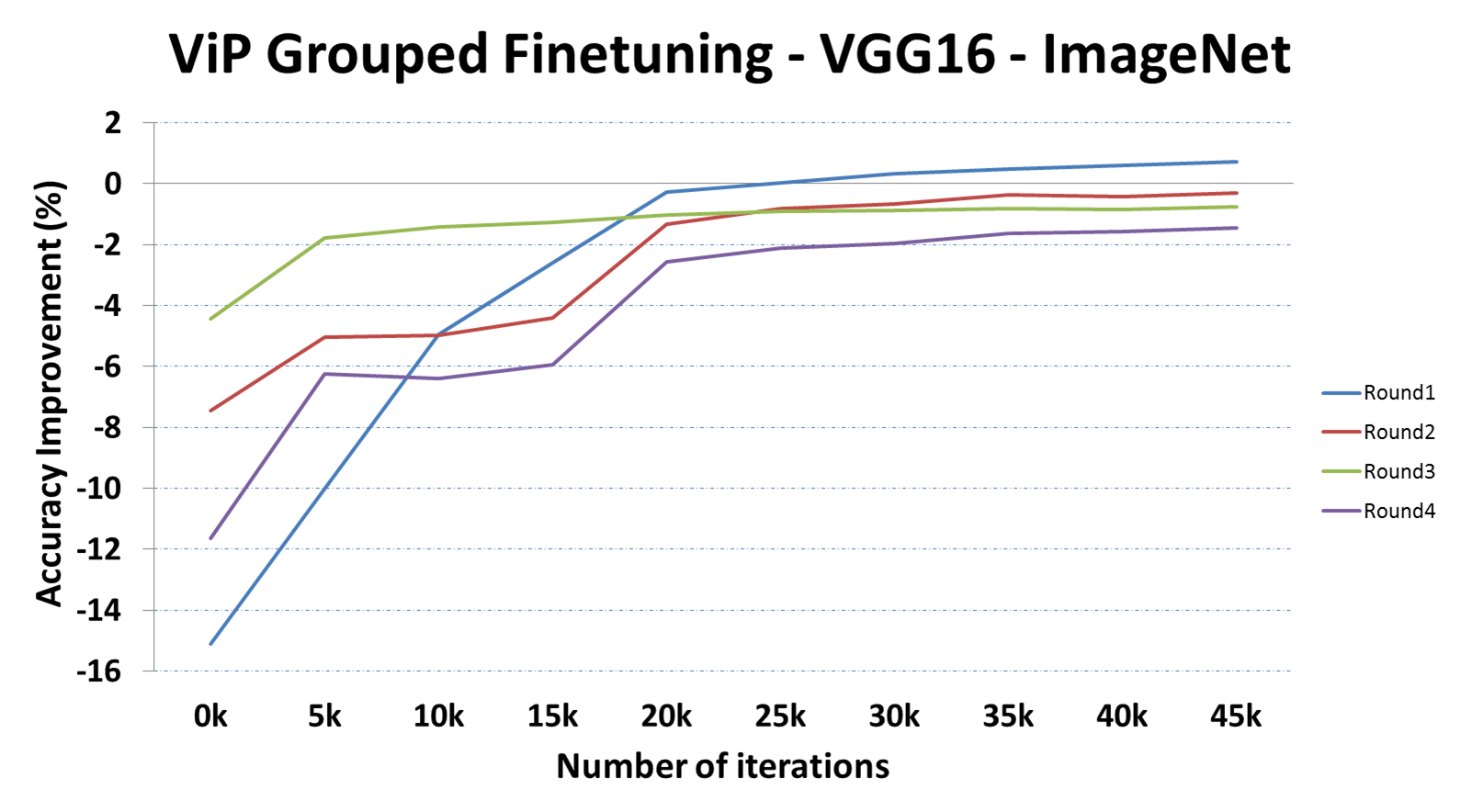}
	\vspace{-15pt}
	\caption{Four rounds of grouped finetuning of VGG16 network using ImageNet dataset.}
	\label{fig:vgg-finetune}

\end{figure}
\subsection{Image Classification}
We first apply the ViP method to speedup and reduce the energy consumption of the image classification task.
\vspace{-10pt}
\subsubsection{Accuracy and Speed}
We experiment with state-of-the-art VGG16 and ResNet-50 models using the ImageNet dataset. The pre-trained models have single-crop top-5 accuracy of 88.5\% and 91.2\%, respectively. We run each of the models 50 times and report the Noise-to-Signal Ratio (NSR) defined as the standard deviation of the measurement divided by mean. NSR measures the relative variance of the experiments and demonstrates the statistical significance of our results. We first apply the ViP method on VGG16 as described in Algorithm \ref{alg:vip} in Section \ref{Method}. We conduct sensitivity analysis to determine the per-layer sensitivity as shown in Fig.\ref{fig:sens-vgg}. The $x$-axis labels provide the names of the layers being interpolated and we explicitly append ``pool" in the name of the layers that are immediately preceding a pooling layer. As we can see, (1) after ViP insertion, different accuracy degradations without finetuning are obtained (shown on $y$-axis in Fig.\ref{fig:sens-vgg}), (2) all layers immediately preceding a pooling layer exhibit the least sensitivity to ViP operation, which was also discovered by \cite{perforatedcnns}. The reason for this is that, although ViP loses information due to interpolation, many of those interpolated values are discarded by the pooling layer, and as a result, ViP has less impact on the final output of the network. And (3) besides the pooling layers, we can see a general trend of decreasing sensitivity when we insert ViP in later-stage layers. This follows the intuition that early perturbations lead to high error on the output when propagating through multiple layers, which is mathematically shown in Eq.\ref{eq:bound}. 
The next step is to do model finetuning with progressively inserted ViP layers. We use grouped finetuning in the case of VGG16 to save training time. Specifically, we have four rounds of finetuning according to the sensitivity of the layers: (1) in round one, we insert ViP after conv layers 13, 12, 10, 7, 2 and 4; (2) in round two, we further insert ViP after conv layers 11, 9 and 8; (3) in round three, we further insert ViP after conv layer 1; (4) in the final round, we insert ViP layers after the remaining conv layers. Each round is initialized with the trained model from the previous round, because this (1) gives slightly higher accuracy than using the baseline model and (2) saves training time. 

Furthermore, we plot the training curve to illustrate how test accuracy recovers during grouped finetuning across four rounds, as shown in Fig.\ref{fig:vgg-finetune}. The zero line indicates the accuracy of the baseline network, and the $y$-axis is the accuracy improvement (degradation if negative) during finetuning. For fair comparison, we use top-5 accuracy for ImageNet throughout the paper as also reported in \cite{perforatedcnns}. The $x$-axis is the number of training iterations. We can see that after the initial insertion of ViP layers, there is a huge drop in accuracy. However, this gradually recovers during the finetuning step and even surpasses the original accuracy in round one. We conjecture that this is similar to the effect observed in \cite{han2016dsd}, where linear interpolation serves as a type of regularization that improves network generalization. 
\begin{figure}[h!]
	\centering
	
	\includegraphics[width=\linewidth,trim=4 4 4 4, clip]{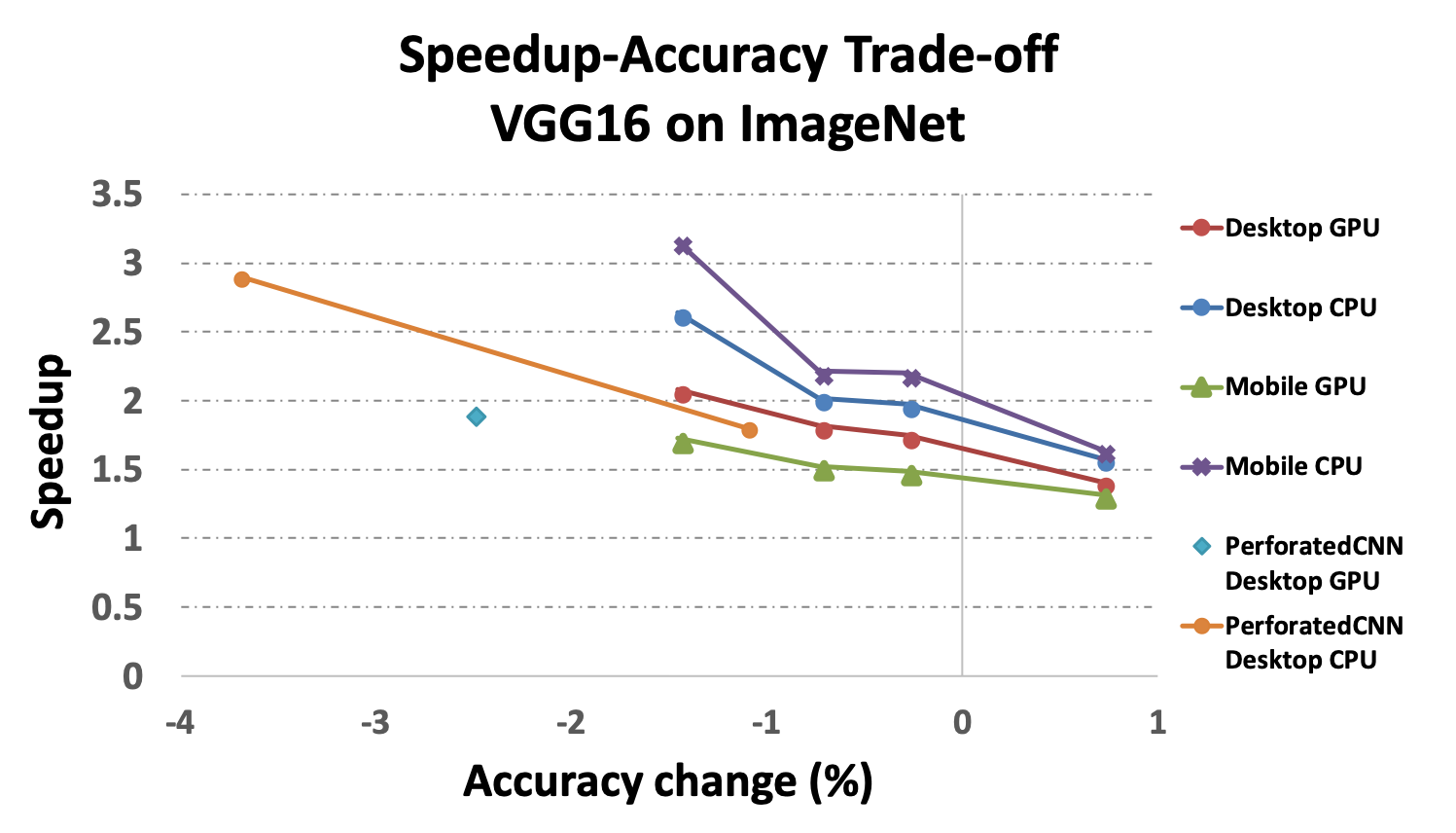}
	\vspace{-15pt}
	\caption{Speedup-Accuracy trade-off obtained by applying ViP on VGG16 model with ImageNet dataset.}
	\label{fig:speedup-vgg}
	\vspace{-10pt}
\end{figure}

After four rounds of grouped finetuning, we obtain four models of different speedup-accuracy trade-offs. 
A positive value for accuracy change means improvement, while a negative value means accuracy drop. Speedup is measured as the ratio of the inference time of the original model over the inference time of the model with ViP. We finetune the model with ViP on the desktop machine, because (1) storage of the mobile platform is insufficient for holding the entire ImageNet dataset, (2) training on desktop machine is significantly faster and the trained model can be evaluated on both desktop and mobile platforms for runtime analysis, and (3) model accuracy is platform-independent, which means once a model is obtained, its test accuracy remains the same on any platform. Accordingly, we can report accuracy and speedup on both desktop and mobile platforms, while we only train the model on the desktop machine once. 

We plot the results in Fig.\ref{fig:speedup-vgg} along with the result of the previous state-of-the-art PerforatedCNNs \cite{perforatedcnns}. Our method achieves \textbf{2.1x} speedup with less than $1.5\%$ accuracy drop, while PerforatedCNNs can theoretically achieve $1.9$x speedup with $2.5\%$ accuracy degradation. The measured speedup of PerforatedCNNs is 2x when considering the reduced memory cost through implicit interpolation in Matlab \cite{perforatedcnns}. In the same way, ViP can also reduce memory transfer cost between layers thanks to the smaller-sized intermediate outputs by using larger-stride convolution. Unfortunately, Caffe does not support implicit interpolation and hence no memory saving of intermediate outputs as pointed out by PerforatedCNNs \cite{perfcaffe}. For fair comparison, we eliminate the effect of memory saving in both implementations and use the theoretical upper-limit for PerforatedCNNs speedup since they did not report speedup on Caffe. We expect ViP method to achieve even higher speedup in implementations that support implicit interpolation which saves memory transfer cost. Comparing ViP and PerforatedCNNs on desktop CPU, we can see that ViP (blue curve) is better than PerforatedCNNs (orange line) in terms of Pareto optimality, because models closer to upper-right corner deliver better trade-off between low accuracy drop and high speedup. In the case of mobile CPU, ViP is able to speed up the CNN by \textbf{3.16x} with less than $1.5\%$ accuracy drop. NSR of VGG16 latency on CPU and GPU is 0.6\% and 0.1\%, respectively, which is negligible and shows that our speedup results are reliable. Besides, what ViP can obtain is a set of models with different speedup-accuracy trade-offs rather than a single configuration, CNN practitioners can pick any of the models in Fig.\ref{fig:speedup-vgg} that meets their need. 
\begin{figure*}[h!]
	\centering
	\includegraphics[width=\textwidth,trim=4 10 4 4, clip]{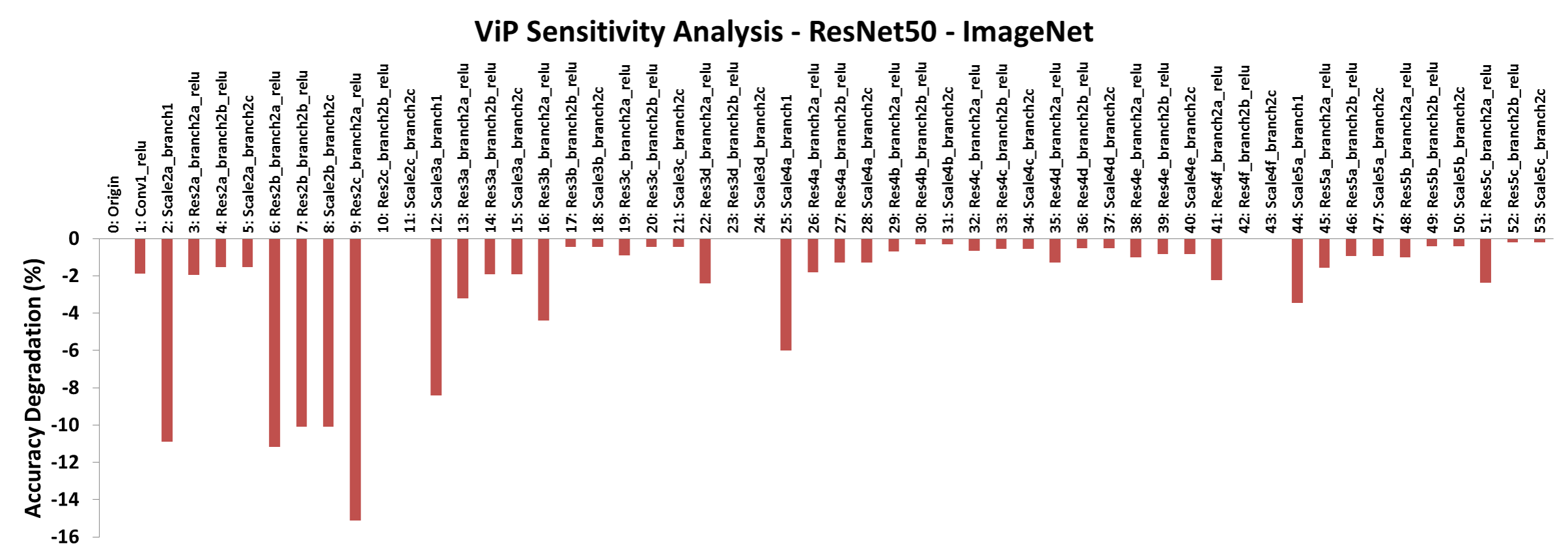}
	\vspace{-15pt}
	\caption{ViP sensitivity analysis of ResNet-50 model under ImageNet dataset. For each of the conv layers, we insert ViP immediately after it, and evaluate the network accuracy without finetuning. The sensitivity is measured as the accuracy degradation.}
	\label{fig:sens-res}
	\vspace{-15pt}
\end{figure*}
Similarly, we apply ViP on ResNet-50 under ImageNet dataset. Fig.\ref{fig:sens-res} shows the results on sensitivity analysis and again we see the trend of decreasing sensitivity in later-stage layers. We have in total 53 conv layers because there are 49 conv layers on the primary branch and four on the bypass branches. Initially, we apply three rounds of grouped finetuning on ResNet-50. However, the final round, consisting of layers with the highest sensitivity, results in a steep accuracy drop, from $-0.7\%$ to $-3.94\%$, we decide to use per-layer finetuning for the 12 layers in the last round to demonstrate the fine-grained progressive change in both accuracy and speed. Fig.\ref{fig:speedup-res} shows the results. As expected, there is a clear trend of increasing speedup with higher accuracy drop when we insert more ViP layers. The speedup of mobile GPU and desktop GPU almost overlaps, and they both achieve \textbf{1.53x} speedup with less than $4\%$ accuracy degradation. Meanwhile, mobile CPU obtains \textbf{2.3x} speedup at same level of accuracy. NSR of Resnet-50 latency on CPU and GPU is 0.5\% and 0.05\%, respectively, which is again negligible and shows that our speedup results are reliable.
\begin{figure}[h!]

	\centering
	\includegraphics[width=\linewidth,trim=4 4 4 4, clip]{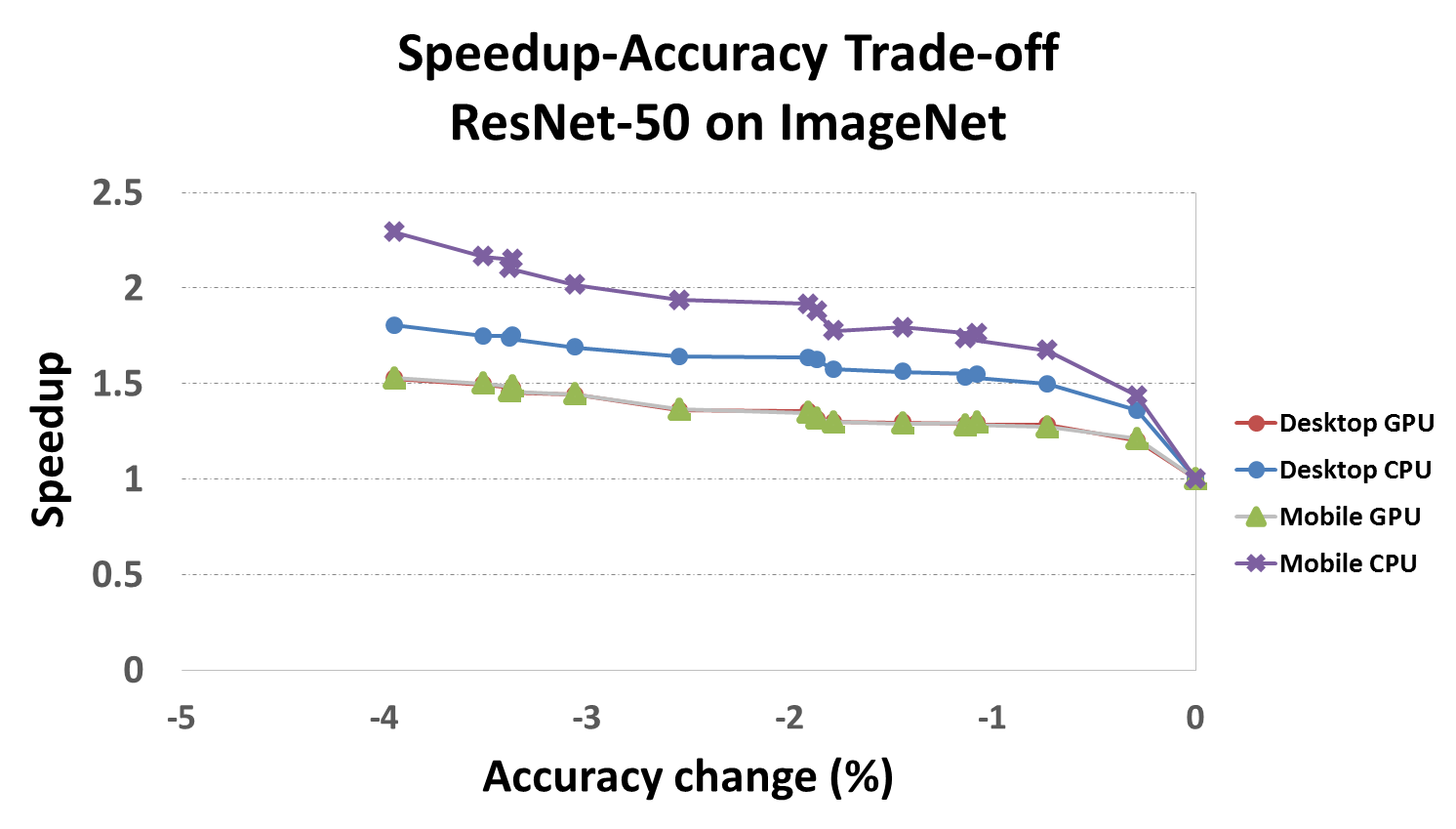}
	\vspace{-15pt}
	\caption{Speedup-Accuracy trade-off obtained by applying ViP on ResNet-50 model with ImageNet dataset.}
	\label{fig:speedup-res}
	\vspace{-5pt}
\end{figure}

Our results on the All-CNN network \cite{allcnn} show a \textbf{1.77x} speedup on the desktop GPU and up to \textbf{3.03x} speedup on the mobile CPU, while the top-1 accuracy drop is within $4\%$. Details are provided in the supplementary material.
\vspace{-10pt}

\begin{figure}[htb!]
	
	\centering
	\includegraphics[width=\linewidth,trim=4 4 4 4, clip]{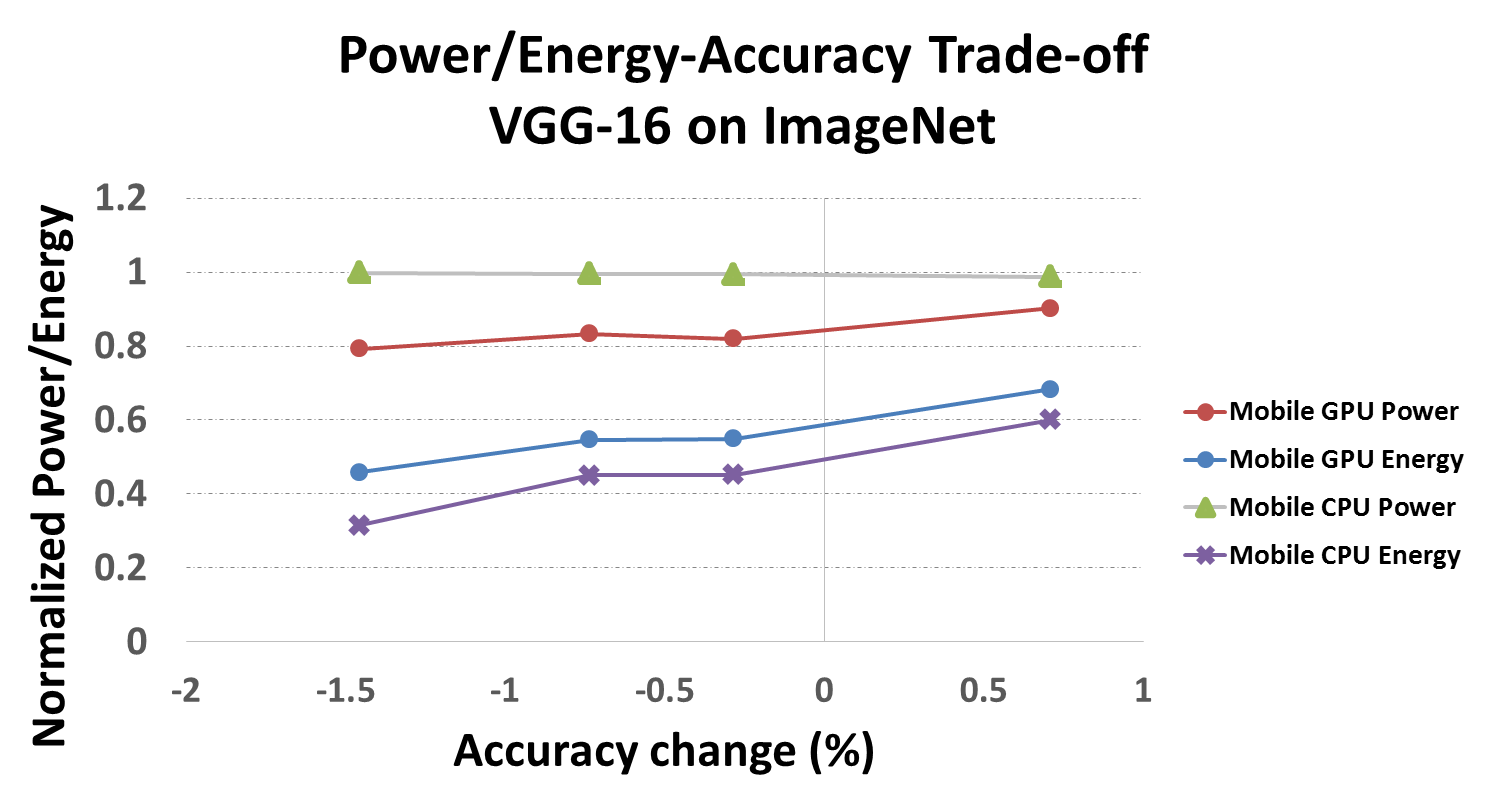}
	\vspace{-15pt}
	\caption{Power/Energy-Accuracy trade-off obtained by applying ViP on VGG16 model with ImageNet dataset.}
	\label{fig:power-vgg}
	\vspace{-10pt}
\end{figure}

\subsubsection{Power and Energy}
More and more mobile apps start to utilize CNNs to improve their image classification and object detection functions. Other than speed, power and energy are the most critical constraints on mobile platforms. Therefore, we further conduct experiments to show how ViP improves the power and energy profile on mobile platforms running CNN.

\begin{figure}[htb!]
	\centering
	\includegraphics[width=\linewidth,trim=4 4 4 4, clip]{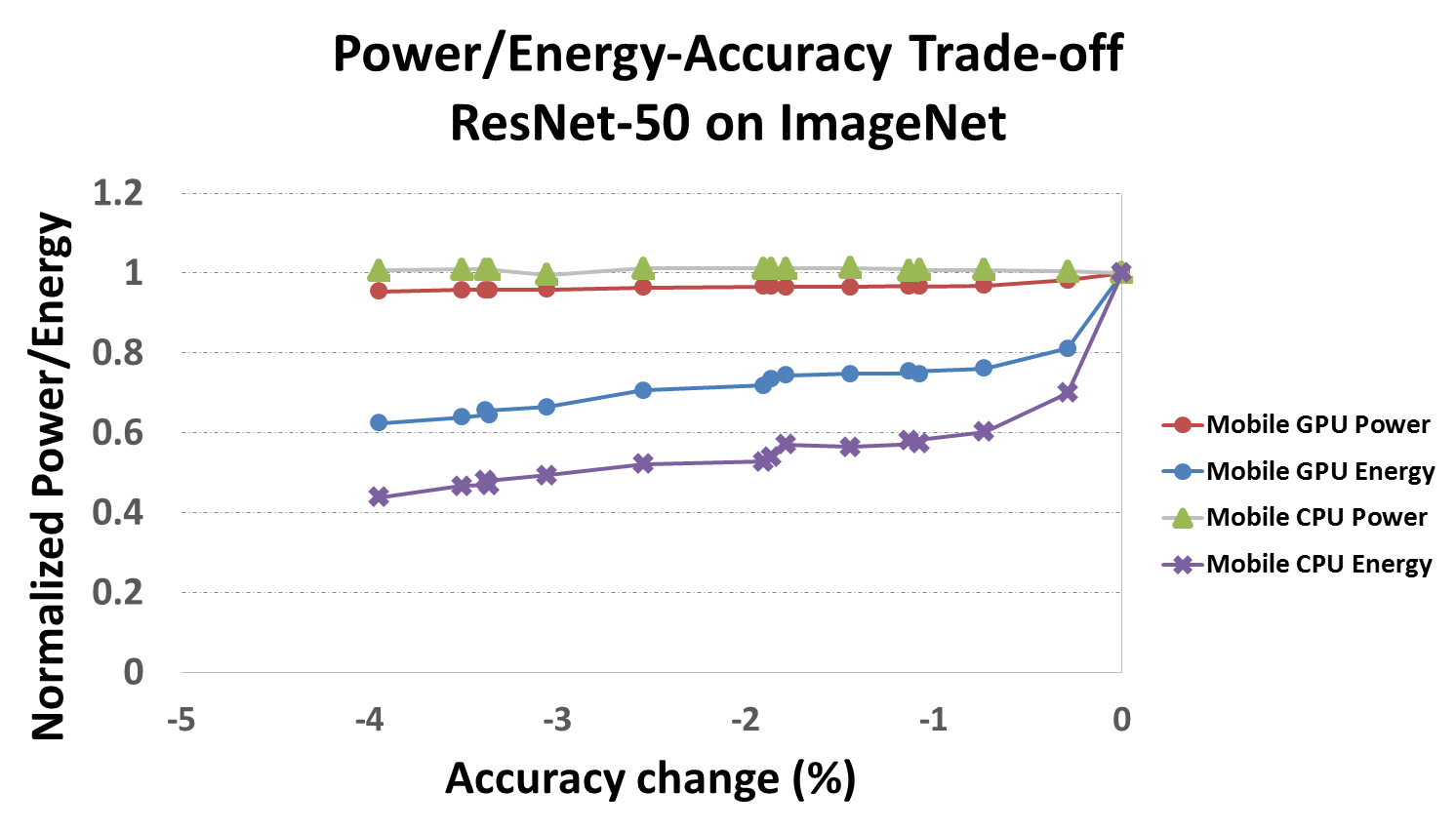}

	\caption{Power/Energy-Accuracy trade-off obtained by applying ViP on ResNet-50 model with ImageNet dataset.}
	\label{fig:power-res}
	\vspace{-10pt}
\end{figure}

We first port both Caffe and our custom ViP layer to Jetson TX1. We use the on-board sensor to measure the power consumption of CNNs with and without ViP, and obtain the energy consumption by multiplying power by CNN latency. We test on all CNNs used previously, \textit{i.e.}, All-CNN, VGG16 and ResNet-50 (Detailed results on All-CNN are included in the supplementary material). Each test is run 50 times and we report the mean power/energy-accuracy trade-off curves in Figures \ref{fig:power-vgg} and \ref{fig:power-res}, respectively. In each of the figures, we show four curves for power and energy consumption of either running on mobile CPU or mobile GPU. With ViP applied, both models show power reduction on mobile GPU, with VGG16 saving up to \textbf{21\%}. VGG16 and ResNet-50 achieve up to \textbf{55\%} and \textbf{38\%} mobile GPU energy reduction, respectively. Furthermore, VGG16 achieves up to \textbf{~70\%} CPU energy reduction while ResNet-50 tops at around \textbf{60\%}. In terms of measurement variance, NSR of ResNet-50 power/energy on mobile CPU and GPU is 3.2\% and 9.8\%, respectively, while the NSR of VGG16 power/energy on mobile CPU and GPU is 2.9\% and 12.1\% respectively. As the variance on CPU is negligible, and also small enough on GPU, with high confidence, ViP saves power and energy on both platforms.
\vspace{-10pt}

\subsubsection{ViP for Compressed CNNs}
To demonstrate that ViP is complementary to other acceleration approaches, we apply ViP to ThiNet \cite{luo2017thinet}, a compressed VGG16 via state-of-the-art filter level pruning method, and show greater speedup and energy saving can be achieved when applying both ViP and network compression. ThiNet is only 6\% the size of the original VGG16 \cite{luo2017thinet} and our measured latency on desktop GPU shows a 2.75x speedup over VGG16 with minor accuracy drop. We apply ViP on top of ThiNet, and three rounds of finetuning are carried out after the sensitivity analysis. Fig.\ref{fig:speedup-thinet} shows the results of speedup and accuracy drop on desktop GPU, relative to ThiNet, after each round of finetuning. We can see that ViP achieves 30\% speedup with 1.3\% accuracy drop and can reach up to 1.9x speedup on top of the already heavily compressed ThiNet. By combining both ViP and compression, we can drastically speedup CNN by a factor of 5.23x. Fig.\ref{fig:power-thinet} shows the normalized power and energy consumption on both mobile CPU and GPU after applying ViP on ThiNet. ViP further reduces the energy consumption of ThiNet by up to 60\% when running on mobile GPU. These results demonstrate that ViP is indeed a complementary method to the existing acceleration approaches, and when we apply both compression and ViP, we can achieve greater speedup and energy saving.
\begin{figure}[h!]
	\centering	
	\includegraphics[width=\linewidth,trim=4 4 4 4, clip]{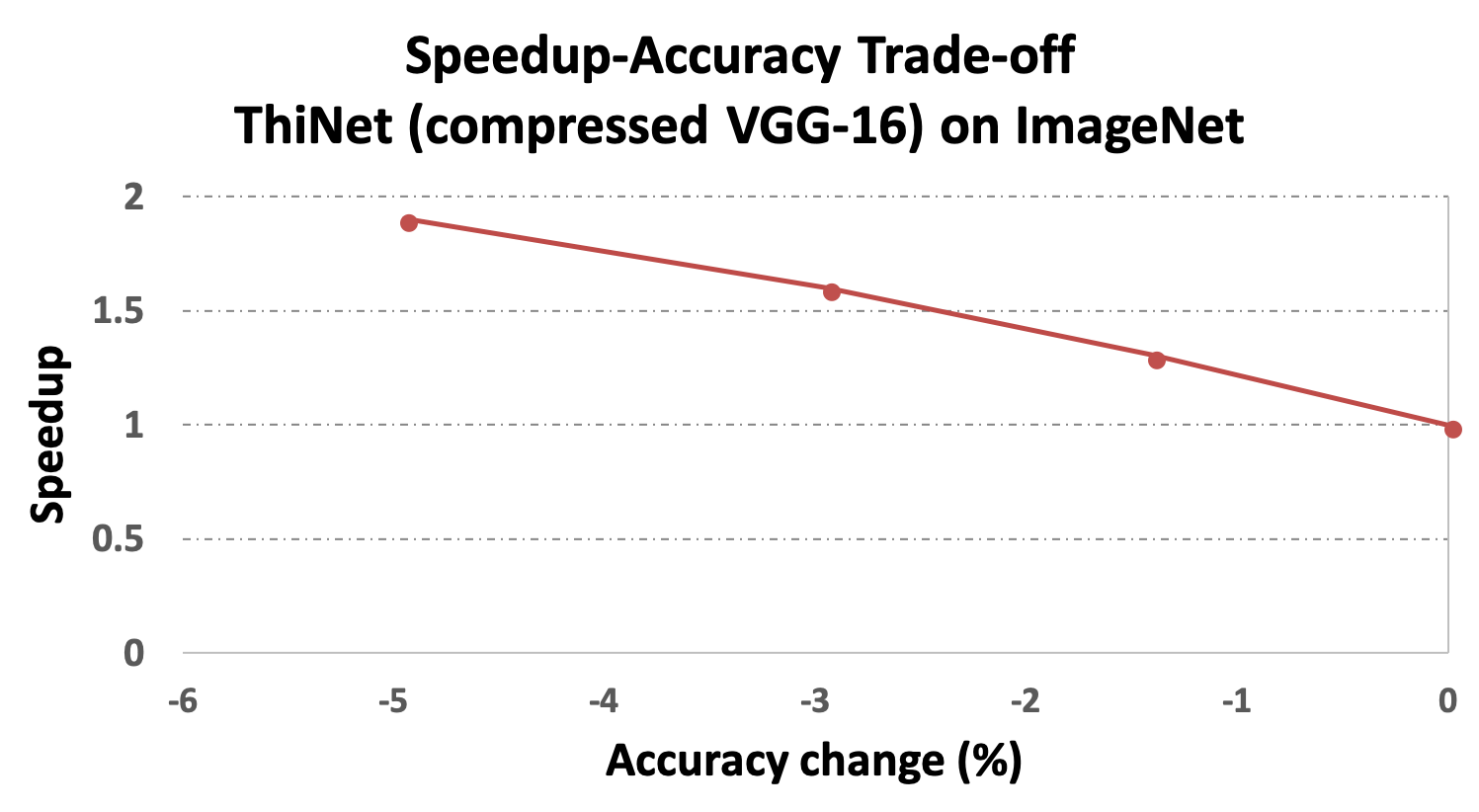}
	\vspace{-15pt}
	\caption{Speedup-Accuracy trade-off obtained by applying ViP on ThiNet (compressed VGG16) model with ImageNet dataset.}
	\label{fig:speedup-thinet}
	\vspace{-10pt}
\end{figure}

\begin{figure}[h!]
	\centering	
	\includegraphics[width=\linewidth,trim=4 4 4 4, clip]{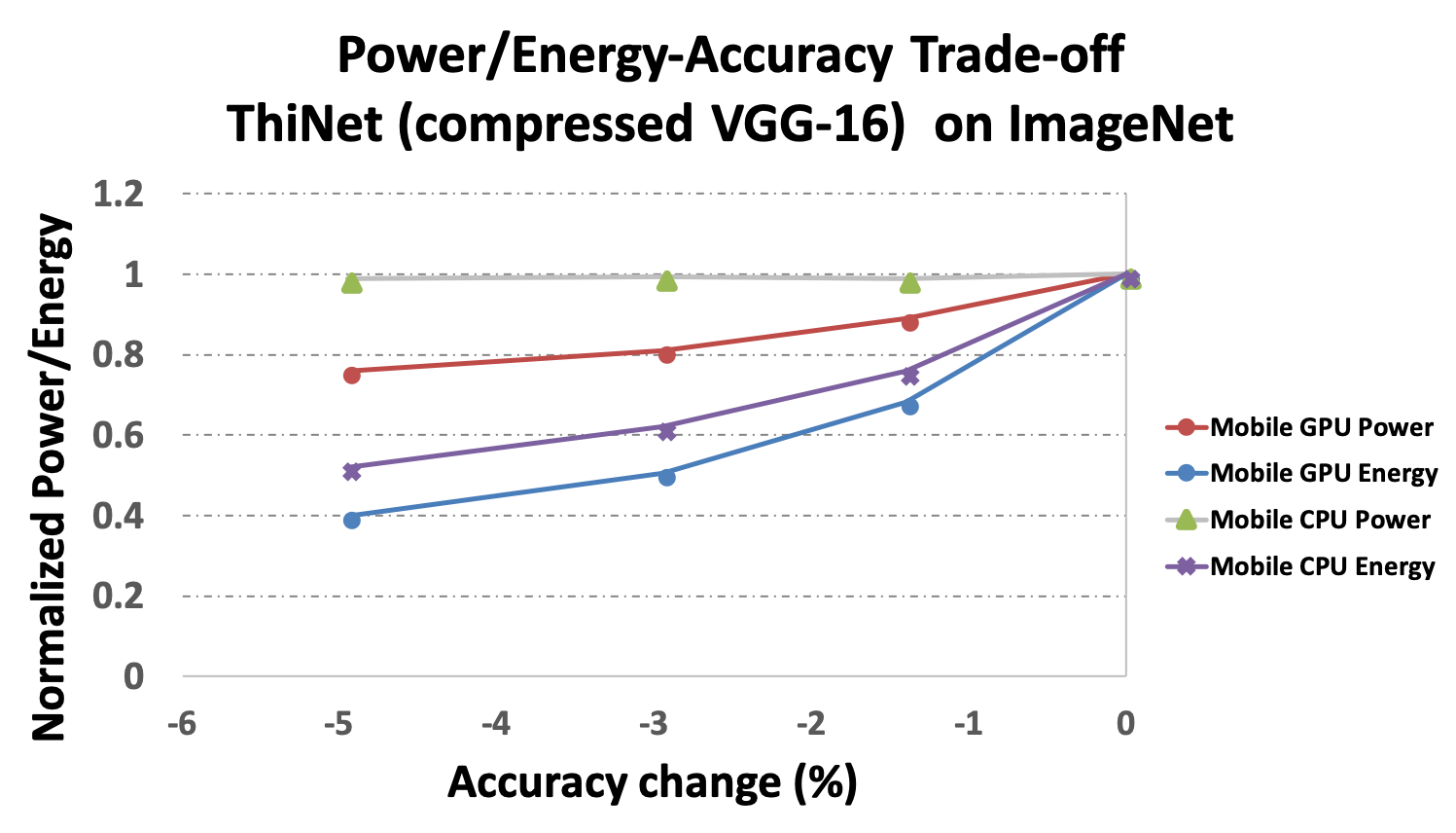}
	\vspace{-15pt}
	\caption{Power/Energy-Accuracy trade-off obtained by applying ViP on ThiNet (compressed VGG16) with ImageNet dataset.}
	\label{fig:power-thinet}
	\vspace{-10pt}
\end{figure}

\subsection{Object Detection}
Much of the prior work on CNN acceleration only studies image classification \cite{wen2016learning,lightnn,lebedev2014speeding}, while object detection is often a more practical and interesting application. Although the two tasks share some common features, object detection has its unique components and challenges, \textit{e.g.}, region proposal, bounding box regression, \textit{etc}. Thus, without experimental results, it is hardly convincing to infer that methods excel on classification can also work well on detection tasks. Accordingly, in this section, we further test ViP on object detection and show that it works across both important tasks.
\vspace{-10pt}
\begin{figure}[htb!]
	\centering
	\includegraphics[width=\linewidth,trim=4 4 4 4, clip]{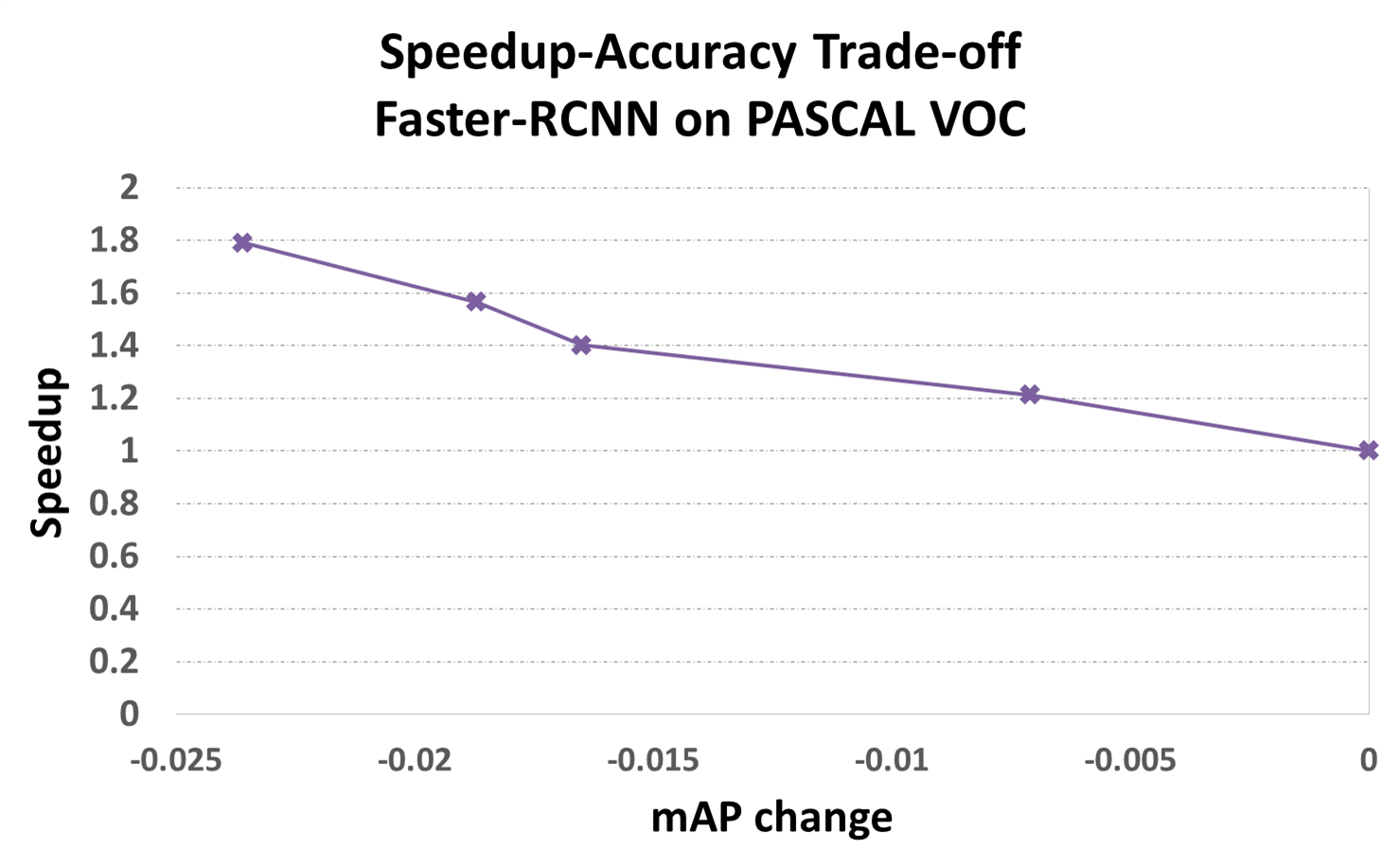}
	\vspace{-10pt}
	\caption{Speedup-Accuracy trade-off obtained by applying ViP on faster-rcnn with VGG16 backbone under PASCAL VOC 2007.}
	\label{fig:speedup-rcnn}
	
\end{figure}

We use the Caffe implementation of the state-of-the-art object detection framework faster-rcnn \cite{faster-rcnn} with PASCAL VOC 2007 dataset, and integrate it with our custom ViP layer. The pre-trained faster-rcnn, with VGG16 as backbone, has an accuracy of 69.5\% mAP. 
We conduct four rounds of grouped finetuning after sensitivity analysis (detailed in supplementary material), and as expected, with more layers followed by ViP operation, we are able to achieve higher speedup but with higher mAP degradation as shown in Fig.\ref{fig:speedup-rcnn}. In the end, we apply ViP to all conv layers and achieve \textbf{1.8x} speedup with $0.025$ mAP drop. 
\vspace{-5pt}
\section{Conclusion}
\label{conclusion}
In this work, we propose the Virtual Pooling (ViP) method that combines downsampling, efficient upsampling and sensitivity-based grouped finetuning, with a provable bound for speeding up CNNs with low accuracy drop. We validate our method extensively on four CNN models, three representative datasets, both desktop and mobile platforms, and on both image classification and object detection tasks. ViP is able to speedup VGG16 by \textbf{2.1x} with less than $1.5\%$ accuracy drop, and speedup faster-rcnn by \textbf{1.8x} with $0.025$ mAP degradation. Combining ViP and model compression leads to a \textbf{5.23x} speedup on VGG16. Furthermore, ViP generates a set of models with different speedup-accuracy trade-offs. This provides CNN practitioners a tool for finding the model best suiting their needs.
\\

This research was supported in part by National Science Foundation CNS Grant No. 1564022. Zhuo Chen acknowledges support from Qualcomm Innovation Fellowship.

{\small
\bibliographystyle{ieee}
\bibliography{egbib}
}

\end{document}